\documentclass{article}
\usepackage{spconf,amsmath,graphicx,hyperref}

\usepackage{scalefnt}
\usepackage{setspace}
\usepackage{romannum}
\usepackage[dvipsnames]{xcolor}
\usepackage{tkz-berge}
\usetikzlibrary{fit,shapes}
\usepackage{calrsfs}
\usepackage{graphics}
\usepackage{epsfig}
\usepackage{mathptmx}
\usepackage{times}
\usepackage{tikz}
\usetikzlibrary{positioning}
\usepackage{amsmath, amssymb}
\usetikzlibrary{arrows}
 \usepackage[T1]{fontenc}
\usepackage{bm, bbm}
\usepackage{graphicx}
\usepackage{subcaption}

\usepackage{kantlipsum}
\usepackage{amsthm}
\usepackage{amsfonts}
\usepackage{multirow}
\usepackage{textcomp}
\usepackage[english]{babel}
\usepackage{float}

\usepackage[section]{algorithm}
\usepackage{algcompatible}
\usepackage{mathtools, cuted}
\usepackage{rotating}
\usepackage{pgfplots}
\pgfplotsset{compat=1.5}
\usepackage{siunitx}
\usepackage{pgfplotstable}
\usepackage{bbold}

\def\BibTeX{{\rm B\kern-.05em{\sc i\kern-.025em b}\kern-.08em
    T\kern-.1667em\lower.7ex\hbox{E}\kern-.125emX}}

\usepackage{enumitem}
\usepackage{comment}
\usepackage{microtype}
\usepackage{booktabs}
\usepackage{hyperref}
\usepackage{amsmath}
\usepackage{amssymb}

\theoremstyle{plain}
\newtheorem{theorem}{Theorem}
\numberwithin{theorem}{section}

\theoremstyle{definition}

\newtheorem{assumption}[theorem]{Assumption}
\theoremstyle{remark}

\DeclareMathOperator*{\argmax}{arg\,max}

\DeclareMathAlphabet{\pazocal}{OMS}{zplm}{m}{n}

\DeclareMathOperator{\Tr}{Trace}
\newcommand\scalemath[2]{\scalebox{#1}{\mbox{\ensuremath{\displaystyle #2}}}}

\newcommand{\SE}{\mathrm{SD}}  

\newcommand{\X}{\pazocal{X}}

\newcommand{\pR}{\pazocal{R}}

\newcommand{\bR}{\mathbb{R}}
\newcommand{\bE}{\mathbb{E}}

\newcommand{\wB}{\widehat{B}}
\newcommand{\wW}{\widehat{W}}
\newcommand{\ww}{\widehat{w}}

\newcommand{\NSR}{\mathrm{NSR}}
\newcommand{\svdeq}{\overset{\mathrm{SVD}}=} 

\newcommand{\mt}{{m, t}}
\newcommand{\thetats}{\theta_t^\star}
\newcommand{\thetahat}{\widehat{\theta}}
\usepackage{dsfont}
\newcommand{\indic}{\mathds{1}}

\newcommand{\nt}{{n, t}}

\newcommand{\NT}{{N, T}}
\newcommand{\pxtc}{x_\nt}
\newcommand{\norm}[1]{\left\lVert#1\right\rVert}

\newcommand{\cred}{\color{red}}
\newcommand{\Pnt}{\Phi_\nt}

\newcommand{\Ynt}{Y_\nt}

\newcommand{\Ent}{H_\nt}
\newcommand{\Vnt}{V_\nt}

\newcommand{\BVnt}{\bar{V}_\nt}

\newcommand{\BVntm}{\bar{V}_{n-1, t}}
\newcommand{\BVmtm}{\bar{V}_{m-1, t}}

                     \makeatletter
\makeatletter
\newcommand{\specificthanks}[1]{\@fnsymbol{#1}}
\makeatother

\title{RepUCB: Representation Learning-Based UCB for Heterogeneous Multi-Task Linear Bandits}
\name{
Jiabin Lin and Shana Moothedath \thanks{J. Lin is with the Department of Computer Science and Technology, Qingdao University, China, and S. Moothedath is with the Department of Electrical and Computer Engineering, Iowa State University, USA.}
}

\address{
}

\begin{document}
%
\maketitle
\begin{abstract}
Multi-task representation learning exploits the shared structure among related tasks by learning a common latent representation, thereby improving sample efficiency. This paper introduces a novel approach to multi-task representation learning in heterogeneous linear bandits.
We consider $T$ concurrent heterogeneous linear bandit tasks, each with feature dimension $d$, whose reward parameters share a common latent representation of dimension $r \ll \min\{d, T\}$, capturing the underlying task relatedness.
We propose RepUCB, a novel Upper Confidence Bound (UCB) algorithm that leverages shared low-rank representations to enhance learning in a sample-efficient manner. Our algorithm first collects data through an exploration phase, estimates the shared representation, and then performs UCB-based learning on our proposed confidence set. We provide theoretical guarantees for the confidence set and prove that the unknown reward parameters lie within the confidence set with high probability. We derive cumulative regret bound and show that the proposed approach achieves $\widetilde{O}(\sqrt{drNT})$, a significant improvement over solving the $T$ tasks independently, resulting in a regret of $\widetilde{O}(dT\sqrt{N})$. We performed numerical simulations to validate the performance of our algorithm for different problem sizes and compared with baseline algorithms.
\end{abstract}
\section{Introduction}
\label{sec:intro}

Multi-task representation learning (MTRL) is an emerging paradigm that leverages the observation that related tasks share common structure, enabling the learning of a unified representation that mitigates data scarcity for individual tasks \cite{thekumparampil2021sample,OurICML, collins2021exploiting}. 
Identifying shared patterns, MTRL improves overall learning performance while reducing data requirements. MTRL has been used with great success in natural language processing (GPT-2, GPT-3, GPT-4, Bert), computer vision (CLIP), and recommendation systems \cite{wang2024lora, bose2025lore, schotthofer2025dynamical}.
MTRL has been well studied within the linear regression framework \cite{collins2021exploiting, du2020few, tripuraneni2021provable}, and more recently, its analysis has been extended to dynamic learning environments, particularly in bandit settings \cite{OurICML, cella2023multi, yang2020impact}. 

Stochastic bandit learning is a well-studied problem in online sequential decision-making, applicable in diverse fields, including robotics, clinical trials, communications, and recommendation systems. 
In each round, the learner chooses an action and receives a stochastic reward from the environment. The objective is to balance exploration and exploitation to maximize cumulative reward over a finite horizon when the reward parameters are initially unknown.

This paper presents a provably efficient online algorithm for distributed linear bandits based on multi-task representation learning, where the reward parameters of $T$ tasks with $d$-dimensional features share a common $r$-dimensional latent representation.
That is, the reward parameters of the $T$ tasks lie in an unknown $r$-dimensional subspace of $\bR^d$, where $r \ll \min\{d, T\}$.
Our goal is to develop an Optimism in the Face of Uncertainty Learning (OFUL) algorithm that exploits the problem structure for efficient learning. By leveraging task dependencies, we aim to achieve a tighter regret bound than solving tasks independently. 
An example of multi-task bandit problems with a low-rank model is recommender systems. For instance, different streaming platforms (tasks) aim to suggest movies (actions) to users to maximize views. While each platform differs, they address the same core problem, leading to related tasks with low-rank reward matrices.

We make the following contributions.

\noindent\emph{\bf Contributions:}
We propose an upper confidence bound (UCB) algorithm based on the  OFUL principle for multi-task bandits using representation learning, referred to as RepUCB.
Our approach utilizes a spectral initialization approach from \cite{matcomp_candes, singh2024noisy} to learn the shared model of the tasks.
Using the shared estimate, we develop a novel approach to construct the confidence set for each task. We prove in Theorem~\ref{l1} that the reward parameters of the tasks lie within the confidence set, with high probability. Our bound is tighter compared to the standard (single task) UCB approach in \cite{abbasi2011improved}, as we utilize the structure of the distributed problem. We prove the regret guarantees for our proposed algorithm in Theorem~\ref{t1}. We show that, under suitable noise-to-signal ratio conditions, the regret over $N$ rounds scales as $\widetilde{O}(\sqrt{drNT})$, significantly outperforming the $\widetilde{O}(dT\sqrt{N})$ regret when solving tasks independently using standard UCB. Here, the $\widetilde{O}$ notation hides the logarithmic factors. 
We evaluated the proposed approach through numerical simulations for varying problem sizes and compared it with baseline methods. The results demonstrate the effectiveness of our approach and its sublinear regret scaling.


\noindent\emph{\bf Related Work:}
Multi-task bandit learning has been investigated in prior studies \cite{OurICML,cella2023multi, yang2020impact, hu2021near}, aiming to address $T$ related tasks of dimension $d$ that share a common $r$ dimensional linear representation within a concurrent learning framework. 
Specifically, \cite{cella2023multi} utilizes a trace-norm relaxation method to estimate the reward matrix. However, the solution obtained from the trace-norm relaxation does not necessarily solve the original problem.
The study in \cite{yang2020impact} builds on \cite{du2020few}, exploring finite and infinite action settings. They propose an explore-then-commit algorithm; however, their approach assumes that the optimal solution is known, with the primary focus on sample complexity.
\cite{qin2022non} considered a non-stationary and infinite action setting; the actions are constrained to be from an ellipsoid, and the chosen actions must be independent and span the entire action space.
\cite{hu2021near} introduced an OFUL-based algorithm similar to ours. Their approach constructs a confidence set, which requires solving a non-convex optimization problem. However, they assume that the optimal solution to the non-conves estimation problem is known. 
Our earlier work, \cite{OurICML}, proposed an alternating gradient descent and minimization algorithm for solving the multi-task bandit problem. 
The episodic algorithm employs a window-based approach, in which it collects data over each window and subsequently updates the reward matrix estimate through alternating steps of gradient descent and minimization. The algorithm relies on independent and identically distributed (i.i.d.) data in both exploration and commit phases, which becomes restrictive specifically during the commit episodes where actions are chosen greedily.
In a recent work \cite{Jiabin_neurips} we proposed an explore-then-commit algorithm for representation learning followed by an OFUL approach to transfer the learned representation to a new target task. 
The focus in \cite{Jiabin_neurips} is transfer learning to a new task.
In contrast to these works, we focus on developing a provably efficient UCB algorithm for multi-task representation learning over $T$ tasks. By leveraging the shared latent representation, the proposed method improves learning efficiency within the source tasks and promotes knowledge transfer across tasks.

\section{Problem Formulation and Notations}
\label{sec:prob}

\noindent{\bf Notations:} For any positive integer $n$, the set $[n]$ denotes $\{1, 2, \cdots, n\}$. For any vector $x$, $\norm{x}$ represents the $\ell_2$ norm, and $|x|$ indicates the element-wise absolute value.
For any matrix $A$, $\norm{A}$ denotes the $2$-norm and $\norm{A}_F$ denotes the Frobenius norm. $\top$ represents the transpose of a matrix or vector. The notation $I_n$ (or sometimes just $I$) represents the $n \times n$ identity matrix, while $e_k$ denotes the $k-$th canonical basis vector.
For basis matrices $B_1$ and $B_2$, we define Subspace Distance (SD) as $\SE(B_1, B_2) := \|(I - B_1 B_1^\top) B_2\|_F$. 

\vspace{1 mm}
\noindent{\bf Problem Setting:}
Let $\X$ represent the action set, with the environment defined by a fixed but unknown reward function $y:\X \rightarrow \bR$. We consider a scenario with $T$ different tasks, each representing a related but distinct bandit problem. In each round $n$, each task $t \in [T]$ independently chooses an action $x_\nt \in \X$. The task $t$ receives a noisy reward $y_\nt \in \bR$ from the environment, where $y_\nt = x_\nt^\top \theta_t^\star + \eta_\nt,$ 
where $\theta_t^\star$ is the fixed but unknown reward parameter for task $t$, while $\eta_\nt$ denotes reward noise. The expected reward is defined as $r_\nt = x_\nt^\top \theta_t^\star$, where $r_\nt = \bE[y_\nt]$. The goal is to choose actions that maximize the cumulative reward $\sum_{t=1}^T \sum_{n=1}^N y_\nt$, which is equivalent to minimizing the cumulative (pseudo) regret
$$
\pR_{N, T} := \sum_{t = 1}^T \sum_{n = 1}^N x_\nt^{\star^\top} \theta_t^\star - x_\nt^\top \theta_t^\star,
$$
where $x_\nt^\star$ denotes the optimal action at round $n$ for task $t$. We define the reward parameter matrix as $\Theta^\star := [\theta_1^\star \cdots \theta_T^\star]$ and assume that the rank($\Theta^\star$) $\ll \min\{d, T\}$. This low-rank structure enables us to take advantage of similarities between different tasks, thus improving our learning efficiency. 
\begin{assumption}\label{assume:B}
(Common Feature Extractor). We assume the existence of a matrix $B^\star \in \bR^{d\times r}$, denoted as the common feature extractor, along with a set of linear coefficient vectors $\{w_t^\star\}_{t=1}^T$. Hence, the reward parameter matrix can be expressed as $\Theta^\star = B^\star W^\star$, where $W^\star=[w_1^\star, \cdots, w^\star_T]$. 
\end{assumption}
This assumption captures the relatedness among tasks and is widely applied in representation learning literature \cite{OurICML,collins2021exploiting,du2020few, tripuraneni2021provable, cella2023multi, yang2020impact, hu2021near,chen2022active}. 
%
Let $\Theta^\star \svdeq B^\star \Sigma V^{\star} := B^\star  W^\star$ denote its reduced (rank $r$) singular value decomposition, where $B^\star \in \bR^{d \times r}$ and $V^{\star^\top} \in \bR^{T \times r}$ are matrices with orthonormal columns, and $\Sigma \in \bR^{r \times r}$ is a diagonal matrix containing non-negative singular values. The matrix $W^\star := \Sigma V^\star$. Let $\sigma_{\max}^\star$ and $\sigma_{\min}^\star$ denote the maximum and minimum singular values of $\Sigma$, respectively, and we define the condition number $\kappa := \frac{\sigma_{\max}^\star}{\sigma_{\min}^\star}$. 
\begin{assumption} \label{assume:Distribution}
Noise $\eta_\nt$s are independent $ 1$-sub-Gaussian random variables and $\|\pxtc\|\leqslant L$ for a constant $L>0$.
\end{assumption}
\begin{assumption}\label{assume:incoherence}
We assume the existence of constants $l$ and $u$, where $0<l\leqslant u$ such that $l\leqslant\|w_t^\star\|_2\leqslant u$ for all $t\in[T]$.
\end{assumption}
Assumption~\ref{assume:Distribution} is a standard data model assumed in the bandit literature \cite{abbasi2011improved}. Assumption~\ref{assume:incoherence} implies the incoherence of the reward matrix $\Theta^\star$, ensuring that its low-rank structure is sufficiently distributed across tasks, as is standard in low-rank estimation \cite{thekumparampil2021sample,collins2021exploiting, tripuraneni2021provable}.
%

\section{RepUCB: Representation Learning-Based UCB Algorithm for Linear Bandits}
\label{sec:Proposed Algorithm}
%
Our approach consists of two key steps: (i)~random exploration followed by estimation of shared representation and (ii)~OFUL update using shared representation estimate to learn task specific parameters. In step~(i), we collect data samples via a random exploration. During each round $n \in [N_1]$, each task 
$t \in [T]$ independently selects a random action $x_\nt$ and observes the corresponding reward based on the chosen action. At the end of the exploration, the goal of our algorithm is to obtain an estimate $(\wB, \wW)$ that minimizes the cost function 
\begin{align}
f(\wB, \wW)=\sum_{n=1}^{N_1} \sum_{t=1}^T \norm{y_\nt - x_\nt^\top \wB \widehat{w}_t}^2,\label{eq:cost}
\end{align}
where $\wB \in \bR^{d \times r}$ and $\wW=[\ww_1, \cdots, \ww_T] \in \bR^{r \times T}$.
%
\begin{algorithm}[t]
    \caption{RepUCB Algorithm for Multi-Task Bandits} 
    \label{alg: UCB}
\begin{algorithmic}[1]
    \STATE {\bfseries Parameters:} Total number of rounds, $N$; Number of rounds for exploration, $N_1$; Multiplier in specifying $\alpha$, $\tilde{C}$; $\widehat{\theta}_{N_1, t} \leftarrow 0$, $\bar{V}_{N_1, t} = \lambda I_d$, $V_{N_1, t} = \lambda I_r$, for all $t \in [T]$
        \STATE For $n \in [N_1]$, $t \in [T]$, randomly select $x_{n,t}$, observe $y_{n,t}$
    \STATE For $t \in [T]$, $Y_t = [y_{1, t}, \cdots, y_{N_1, t}]^\top$, $\Phi_t = [x_{1, t}, \cdots, x_{N_1, t}]^\top$ 
    \STATE {\bfseries Estimating $\wB$ from the random exploration}\label{step:init}
    \STATE Set $\alpha = \frac{\tilde{C}}{N_1 T} \sum_{n=1, t=1}^{N_1, T} y_{n, t}^2$
    \STATE $Y_{t, trunc}(\alpha) := Y_t \circ \indic_{\{|Y_t| \leqslant \sqrt{\alpha}\}}$
    \STATE $\widehat{\Theta}_0 := \frac{1}{N_1} \sum_{t=1}^T \Phi_t^\top Y_{t, trunc}(\alpha) e_t^\top$\label{line:Bhat}
    \STATE Set $\wB \leftarrow \text{top-}r\text{-singular-vectors of} \; \widehat{\Theta}_0$\label{step:B_init}
    \STATE {\bfseries UCB learning for $N_1+1, \ldots, N$ steps}\label{step:ucb}
    \FOR{$n \leftarrow N_1 + 1, \cdots, N$}
        \FOR{$t \leftarrow 1, \cdots, T$}
            \STATE Construct the confidence ellipsoid $\beta_{n-1, t}$ by Eq.~\eqref{confidence interval}
            \STATE $(x_\nt, \tilde{\theta}_\nt) = \argmax_{x \in \X_\nt, \theta \in \beta_{n-1, t}}\langle x, \theta \rangle$
            \STATE Choose the action $x_\nt$ and observe the reward $y_\nt$
            \STATE Update $\BVnt = \BVntm + \pxtc \pxtc^\top$, 
            \STATE Update $\scalemath{0.95}{\ww_\nt = (\wB^\top \BVnt \wB)^{-1} \sum_{m=N_1+1}^n \wB^\top x_\mt y_\mt}$
            \STATE Update $\thetahat_\nt = \wB \ww_\nt$
        \ENDFOR
    \ENDFOR
\end{algorithmic}
\end{algorithm}
We estimate the shared representation $\wB$ as the top $r$ left singular vectors of a matrix $\widehat{\Theta}_0$ which is defined in line~\ref{line:Bhat} of Algorithm~\ref{alg: UCB}.
Here, we set the multiplier as $\widetilde{C} = 9 \kappa^2 \mu^2$, with $\kappa$ and $\mu$ being functions of $\Theta^\star$ and hence unknown. Using Assumption~\ref{assume:incoherence}, we can set $\kappa^2 \mu^2$ by an estimate of its lower bound, $T \max_t \frac{\|\widehat{\theta}_t\|^2}{\|\widehat{\Theta}\|_F^2}$, with $\|\widehat{\theta}_t\|^2 = \frac{1}{N_1}\sum_{n=1}^{N_1} y_{n, t}^2$ and $\|\widehat{\Theta}\|_F^2 = \frac{1}{N_1} \sum_{t=1}^T \sum_{n=1}^{N_1} y_{n,t}^2$.
Our proposed UCB algorithm integrates the $\wB$ estimate with online observations to construct confidence bounds for action selection. 
The development of the associated UCB algorithm, an estimator, and the analysis of their theoretical guarantees form the main focus of this paper and are described below.
We refer the reader to \cite{matcomp_candes, singh2024noisy, Jiabin_neurips} for details on representation learning.
%

%
%
\subsection{Construction of Confidence Sets}
This section details the construction of the confidence set for each task $t$. The reward parameter for task $t$ is given by $\theta_t^\star=B^\star w_t^\star$. 
%
In each round $n \in [N_1+1, N]$, each task $t \in [T]$ constructs a confidence ellipsoid $\beta_{n-1, t}$ that contains the unknown reward parameter $\theta_t^\star$ with high probability. It then computes an optimistic estimate $\tilde{\theta}_\nt = \argmax_{\theta \in \beta_{n-1, t}} (\max_{x \in \X} x^\top \theta)$ and selects the action $x_\nt = \argmax_{x \in \X} x^\top \tilde{\theta}_\nt$. Equivalently, task $t$ chooses the pair $(x_\nt, \tilde{\theta}_\nt) = \argmax_{x \in \X, \theta \in \beta_{n-1, t}} \langle x, \theta \rangle$, maximizing the expected reward collaboratively. Upon selecting the action $x_\nt$, task $t$ observes the corresponding reward $y_\nt$ and updates its confidence ellipsoid using Eq.~\eqref{confidence interval}. 

For all $n \in [N_1+1, N]$, let $\Pnt := [x_{N_1+1, t} \cdots x_\nt]^\top$, $\Ynt = [y_{N_1+1, t} \cdots y_\nt]^\top$, and $\Ent = [\eta_{N_1+1, t} \cdots \eta_\nt]^\top$. For  $\lambda > 0$, define the matrices $\Vnt=\lambda I + \wB^\top \Pnt^\top \Pnt \wB$ and $\BVnt=\lambda I + \Pnt^\top \Pnt$. Here $\Vnt=\wB^\top \BVnt \wB$. We prove the following.
\begin{theorem} \label{l1}
Suppose Assumptions~\ref{assume:B}, \ref{assume:Distribution} and \ref{assume:incoherence} hold. Pick $\delta_0 < 0.1$ and any $\delta > 0$. If $N_1 T > C \mu^2 \kappa^2 \left(d r \frac{\kappa^2}{\delta_0^2} + \frac{d}{\delta_0^2} \NSR\right)$, then for all tasks $t \in [T]$ and $n \in [N_1+1, N]$, it is guaranteed with probability at least $1 - \delta - \exp(- c (d + T))$ that $\thetats$ is contained within the set
\begin{align}
&\beta_\nt = \left\{\theta \in \bR^d: \|\thetahat_\nt - \theta\|_{\BVnt} \leqslant (1+\delta_0)\sqrt{\lambda}\mu\sqrt{\frac{r}{T}}\sigma_{\max}^\star {\cred +} \right.\nonumber\\
&\scalemath{0.9}{\hspace{-3 mm}\left. \sigma \sqrt{2 \log\big(\frac{\det(\Vnt)^{\frac{1}{2}} \det(\lambda I)^{-\frac{1}{2}}}{\delta}\big)} + 2\sqrt{n-N_1}L\delta_0\mu\sqrt{\frac{r}{T}}\sigma_{\max}^\star\right\}} \hspace{-2 mm}\label{confidence interval}. 
\end{align}
Furthermore, if $N_1 T > C \mu^2 \kappa^2 L^2 d(N-N_1)\left(r\kappa^2+\NSR\right)$, then with probability at least $1 - \delta - \exp(- c (d + T))$, for all $t \in [T]$ and $n \geqslant N_1+1$, $\thetats$ is contained within the set
\begin{align*}
\scalemath{0.9}{\beta_\nt^\prime} &= \scalemath{0.9}{\left\{\theta \in \bR^d: \|\thetahat_\nt - \theta\|_{\BVnt} \leqslant \sigma \sqrt{r\log\big(\frac{1+(n-N_1)L^2 / \lambda}{\delta}\big)}\right.}\\
&\left. + \big(\sqrt{\lambda}+\frac{\sqrt{\lambda}}{L\sqrt{N-N_1}}+2\big)\mu\sqrt{\frac{r}{T}}\sigma_{\max}^\star\right\}. 
\end{align*}
\end{theorem}
\begin{proof}
Let $\ww_\nt$ be the least squares estimate of $w_t^\star$ with $\ell^2$ regularization, where the regularization parameter $\lambda > 0$. 
\begin{align*}
\ww_\nt &= \Vnt^{-1} (\Pnt \wB)^\top \Ynt = \Vnt^{-1} (\Pnt \wB)^\top (\Pnt B^\star w_t^\star + \Ent) \\
&= Vnt^{-1} \wB^\top \Pnt^\top \Ent + \Vnt^{-1} \wB^\top \Pnt^\top \Pnt \wB \wB^\top B^\star w_t^\star \\
&+ \Vnt^{-1} \wB^\top \Pnt^\top \Pnt (I - \wB \wB^\top) B^\star w_t^\star\\
&= \Vnt^{-1} \wB^\top \Pnt^\top \Ent + \wB^\top B^\star w_t^\star - \lambda \Vnt^{-1} \wB^\top B^\star w_t^\star \\
&+ \Vnt^{-1} \wB^\top \Pnt^\top \Pnt (I - \wB \wB^\top) B^\star w_t^\star.
\end{align*}
By multiplying $\wB$ on both sides, we derive
\begin{align}
\wB \ww_\nt &= \wB \Vnt^{-1} \wB^\top \Pnt^\top \Ent + \wB \wB^\top B^\star w_t^\star - \lambda \wB \Vnt^{-1} \wB^\top B^\star w_t^\star \nonumber\\
&+ \wB \Vnt^{-1} \wB^\top \Pnt^\top \Pnt (I - \wB \wB^\top) B^\star w_t^\star \nonumber\\
&=\wB \Vnt^{-1} \wB^\top \Pnt^\top \Ent + B^\star w_t^\star+ (\wB \wB^\top - I) B^\star w_t^\star \nonumber\\
&+\wB \Vnt^{-1} \wB^\top \Pnt^\top \Pnt (I - \wB \wB^\top) B^\star w_t^\star - \lambda \wB \Vnt^{-1} \wB^\top B^\star w_t^\star.\label{eq:Beforez}
\end{align}
Consider a vector $z \in \bR^d$. Multiplying both sides of Eq.~\eqref{eq:Beforez},
\begin{align*}
&z^\top \wB \ww_\nt - z^\top B^\star w_t^\star\\
&= z^\top \wB \Vnt^{-1} \wB^\top \Pnt^\top \Ent - \lambda z^\top \wB \Vnt^{-1} \wB^\top B^\star w_t^\star \\
&+ z^\top \wB \Vnt^{-1} \wB^\top \Pnt^\top \Pnt (I - \wB \wB^\top) B^\star w_t^\star + z^\top (\wB \wB^\top - I) B^\star w_t^\star \\
&= \langle (z^\top \wB)^\top, \wB^\top \Pnt^\top \Ent \rangle_{\Vnt^{-1}} - \lambda \langle (z^\top \wB)^\top, \wB^\top B^\star w_t^\star \rangle_{\Vnt^{-1}} \\
&+ z^\top (\wB \wB^\top - I) B^\star w_t^\star + \langle \wB^\top z, \wB^\top \Pnt^\top \Pnt (I - \wB \wB^\top) B^\star w_t^\star \rangle_{\Vnt^{-1}}.
\end{align*}
Analyzing the absolute value, the upper bound is given as
\begin{align}
&|z^\top \wB \ww_\nt - z^\top B^\star w_t^\star|\nonumber\\
&\leqslant\|(z^\top \wB)^\top\|_{\Vnt^{-1}} \|\wB^\top \Pnt^\top \Pnt (I - \wB \wB^\top) B^\star w_t^\star\|_{\Vnt^{-1}}\nonumber\\
&+ |z^\top (\wB \wB^\top - I) B^\star w_t^\star| + \|(z^\top \wB)^\top\|_{\Vnt^{-1}} \|\wB^\top \Pnt^\top \Ent\|_{\Vnt^{-1}}\nonumber\\
&+\lambda \|(z^\top \wB)^\top\|_{\Vnt^{-1}} \|\wB^\top B^\star w_t^\star\|_{\Vnt^{-1}}.\label{Eq1}
\end{align}
Now we bound each of the terms in Eq.~\eqref{Eq1}. To upper bound $\|\wB^\top \Pnt^\top \Ent\|_{\Vnt^{-1}}$, we apply the method from Theorem~1 in \cite{abbasi2011improved}, which gives us with probability at least $1 - \delta$,
\begin{align}
\|\wB^\top \Pnt^\top \Ent\|_{\Vnt^{-1}}^2 \leqslant 2 \sigma^2 \log\big(\frac{\det(\Vnt)^{\frac{1}{2}} \det(\lambda I)^{-\frac{1}{2}}}{\delta}\big).\hspace{-2 mm}\label{eq:z_term_1}
\end{align}
To upper bound the term $\|\wB^\top B^\star w_t^\star\|_{\Vnt^{-1}}$, we have 
\begin{align}
&\|\wB^\top B^\star w_t^\star\|_{\Vnt^{-1}}^2 = w_t^{\star^\top} B^{\star^\top} \wB \Vnt^{-1} \wB^\top B^\star w_t^\star \nonumber\\
&\leqslant \|\wB^\top B^\star w_t^\star\|_2^2 \|\Vnt^{-1}\|_2 \leqslant \|\wB^\top\|_2^2 \|B^\star w_t^\star\|_2^2 \|\Vnt^{-1}\|_2\nonumber\\
&\leqslant \frac{\|B^\star w_t^\star\|^2}{\lambda_{\min}(\Vnt)}\leqslant\frac{1}{\lambda}\|w_t^\star\|^2.\label{eq:z_term_2}
\end{align}
To upper bound the term $\|\wB^\top \Pnt^\top \Pnt (I - \wB \wB^\top) B^\star w_t^\star\|_{\Vnt^{-1}}$, we have with probability at least $1-\exp(-c(d+T))$,
\begin{align}
&\|\wB^\top \Pnt^\top \Pnt (I - \wB \wB^\top) B^\star w_t^\star\|_{\Vnt^{-1}} \leqslant\|\Pnt (I - \wB \wB^\top) B^\star w_t^\star\| \nonumber \\
&\leqslant \|\Pnt\| \|(I - \wB \wB^\top) B^\star\| \|w_t^\star\| \leqslant\sqrt{n-N_1}L\delta_0\|w_t^\star\|,\label{new_proof3}
\end{align}
where the first inequality follows from $\Pnt \wB (\lambda I + \wB^\top \Pnt^\top \Pnt \wB)^{-1} \wB^\top \Pnt^\top \leqslant I$. The second inequality follows from the Cauchy–Schwarz inequality. The last inequality follows from $\|\Pnt\|\leqslant\sqrt{n-N_1} L$ and Proposition~\ref{p1}. Here, $n$ denotes rounds from $\{N_1+1,\ldots, N\}$. Consider $z = \BVnt (\wB \ww_\nt - B^\star w_t^\star)$. To upper bound $|z^\top (\wB \wB^\top - I) B^\star w_t^\star|$, we have with probability at least $1-\exp(-c(d+T))$,
\begin{align}
&|z^\top (\wB \wB^\top - I) B^\star w_t^\star|=|(\wB \ww_\nt - B^\star w_t^\star)^\top \BVnt (\wB \wB^\top - I) B^\star w_t^\star|\nonumber \\
&=\langle \wB \ww_\nt - B^\star w_t^\star, (\wB \wB^\top - I) B^\star w_t^\star\rangle_{\BVnt}\nonumber\\
&\leqslant \scalemath{0.95}{\|\wB \ww_\nt - B^\star w_t^\star\|_{\BVnt} \|(\wB \wB^\top - I) B^\star w_t^\star\|_{\BVnt}}\nonumber\\
&\leqslant \scalemath{0.95}{\sqrt{\lambda}\|\wB \ww_\nt - B^\star w_t^\star\|_{\BVnt}\|(\wB \wB^\top-I) B^\star w_t^\star\|}\nonumber\\
&+\scalemath{0.95}{\|\wB \ww_\nt - B^\star w_t^\star\|_{\BVnt} \|\Pnt (\wB \wB^\top-I) B^\star w_t^\star\|}\nonumber\\
&\leqslant \scalemath{0.86}{(\sqrt{\lambda}+\sqrt{n-N_1}L)\delta_0\|w_t^\star\|\|\wB \ww_\nt - B^\star w_t^\star\|_{\BVnt}.}\label{eq:z_term_5}
\end{align}
To determine the upper bound of $\|(z^\top \wB)^\top\|_{\Vnt^{-1}}^2$, we have 
\begin{align}
&\|\wB^\top z\|_{\Vnt^{-1}}^2 = z^\top \wB \Vnt^{-1} \wB^\top z\nonumber
\end{align}
\begin{align}
&=(\wB \ww_\nt - B^\star w_t^\star)^\top \BVnt \wB \Vnt^{-1} \wB^\top \BVnt (\wB \ww_\nt - B^\star w_t^\star)\nonumber\\
&=(\wB \ww_\nt - B^\star w_t^\star)^\top \BVnt^{\frac{1}{2}} \BVnt^{\frac{1}{2}} \wB \Vnt^{-1} \wB^\top \BVnt^{\frac{1}{2}} \BVnt^{\frac{1}{2}} (\wB \ww_\nt - B^\star w_t^\star)\nonumber \\
&\leqslant\|\wB \ww_\nt - B^\star w_t^\star\|_{\BVnt}^2 \|\BVnt^{\frac{1}{2}} \wB \Vnt^{-1} \wB^\top \BVnt^{\frac{1}{2}}\|\nonumber \\
&=\|\wB \ww_\nt - B^\star w_t^\star\|_{\BVnt}^2,\label{eq:z_term_3}
\end{align}
where the last step is derived from $\|\BVnt^{\frac{1}{2}} \wB \Vnt^{-1} \wB^\top \BVnt^{\frac{1}{2}}\|=\|\BVnt^{\frac{1}{2}} \wB \Vnt^{-\frac{1}{2}}\|^2=\lambda_{\max}(\Vnt^{-\frac{1}{2}} \wB^\top \BVnt^{\frac{1}{2}} \BVnt^{\frac{1}{2}} \wB \Vnt^{-\frac{1}{2}})=\lambda_{\max}(\Vnt^{-\frac{1}{2}} \Vnt \Vnt^{-\frac{1}{2}})=\lambda_{\max}(\Vnt^{-\frac{1}{2}} \Vnt^{\frac{1}{2}} \Vnt^{\frac{1}{2}} \Vnt^{-\frac{1}{2}})=1$. To bound $|z^\top \wB \ww_\nt - z^\top B^\star w_t^\star|$, we have
\begin{align}
|z^\top \wB \ww_\nt - z^\top B^\star w_t^\star| &= |z^\top (\wB \ww_\nt - B^\star w_t^\star)|\nonumber \\
&=(\wB \ww_\nt - B^\star w_t^\star)^\top \BVnt (\wB \ww_\nt - B^\star w_t^\star)\nonumber \\
&=\|\wB \ww_\nt - B^\star w_t^\star\|_{\BVnt}^2.\label{eq:z_term}
\end{align}
Substituting Eqs.~\eqref{eq:z_term_1}-\eqref{eq:z_term} in Eq.~\eqref{Eq1} gives
\begin{align*}
&\|\wB \ww_\nt - B^\star w_t^\star\|_{\BVnt}^2 \leqslant \|\wB \ww_\nt - B^\star w_t^\star\|_{\BVnt}\Big(\sqrt{\lambda}\|w_t^\star\|\\
&+ \scalemath{0.9}{\sigma \sqrt{2 \log\big(\frac{\det(\Vnt)^{\frac{1}{2}} \det(\lambda I)^{-\frac{1}{2}}}{\delta}\big)} +(\sqrt{\lambda}+2\sqrt{n-N_1}L)\delta_0\|w_t^\star\|\Big)}.
\end{align*}
Thus, by using Assumption~\ref{assume:incoherence}, with probability at least $1 - \delta - \exp(- c (d + T))$, 
\begin{align*}
&\|\wB \ww_\nt - B^\star w_t^\star\|_{\BVnt} \leqslant \sigma \sqrt{2 \log\Big(\frac{\det(\Vnt)^{\frac{1}{2}} \det(\lambda I)^{-\frac{1}{2}}}{\delta}\Big)}\\
&+(1+\delta_0)\sqrt{\lambda}\mu\sqrt{\frac{r}{T}}\sigma_{\max}^\star+2\sqrt{n-N_1}L\delta_0\mu\sqrt{\frac{r}{T}}\sigma_{\max}^\star.
\end{align*}
Furthermore, $\det(\Vnt) = \det(\lambda I + \wB^\top \Pnt^\top \Pnt \wB) = \prod_{i = 1}^r (\lambda + \lambda_i)$, where $\lambda_i$ denote the eigenvalue values of the positive semi-definite matrix $\displaystyle{\wB^\top \Pnt^\top \Pnt \wB}$. Given that 
\begin{align*}
\lambda_i &\leqslant \Tr(\wB^\top \Pnt^\top \Pnt \wB) =\Tr(\sum_{m=N_1+1}^n \wB^\top x_\mt x_\mt^\top \wB) \\
&=\sum_{m=N_1+1}^n \Tr((\wB^\top x_\mt) (\wB^\top x_\mt)^\top) \\
&=\sum_{m=N_1+1}^n (\wB^\top x_\mt)^\top (\wB^\top x_\mt) =\sum_{m=N_1+1}^n \|\wB^\top x_\mt\|_2^2 \\
&\leqslant\sum_{m=N_1+1}^n \|x_\mt\|_2^2 \leqslant (n-N_1) L^2
\end{align*}
we conclude that 
\begin{align}
\det(\Vnt) &\leqslant \prod_{i = 1}^r (\lambda+(n-N_1)L^2) =(\lambda+(n-N_1)L^2)^r \nonumber \\
&=\lambda^r \big(1 + \frac{(n-N_1)L^2}{\lambda}\big)^r. \label{eq:det}
\end{align}
By setting $\delta_0=\frac{1}{\sqrt{n-N_1}L}$, it is essential to ensure that $N_1 T > C \mu^2 \kappa^2 L^2 d(N-N_1)\left(r\kappa^2+\NSR\right)$. Consequently, we show that with probability at least $1 - \delta - \exp(- c (d + T))$, 
\begin{align*}
\|\wB \ww_\nt - B^\star w_t^\star\|_{\BVnt} &\leqslant \sigma\sqrt{r\log\frac{1+(n-N_1)L^2/\lambda}{\delta}} \\
&+ \big(\sqrt{\lambda}+\frac{\sqrt{\lambda}}{L\sqrt{n-N_1}}+2\big)\mu\sqrt{\frac{r}{T}}\sigma_{\max}^\star.
\end{align*}
Thus, we complete the proof. 
\end{proof}
%
\begin{figure*}[t]
\centering
\subcaptionbox{\label{fig:0}}{\includegraphics[scale=0.15]{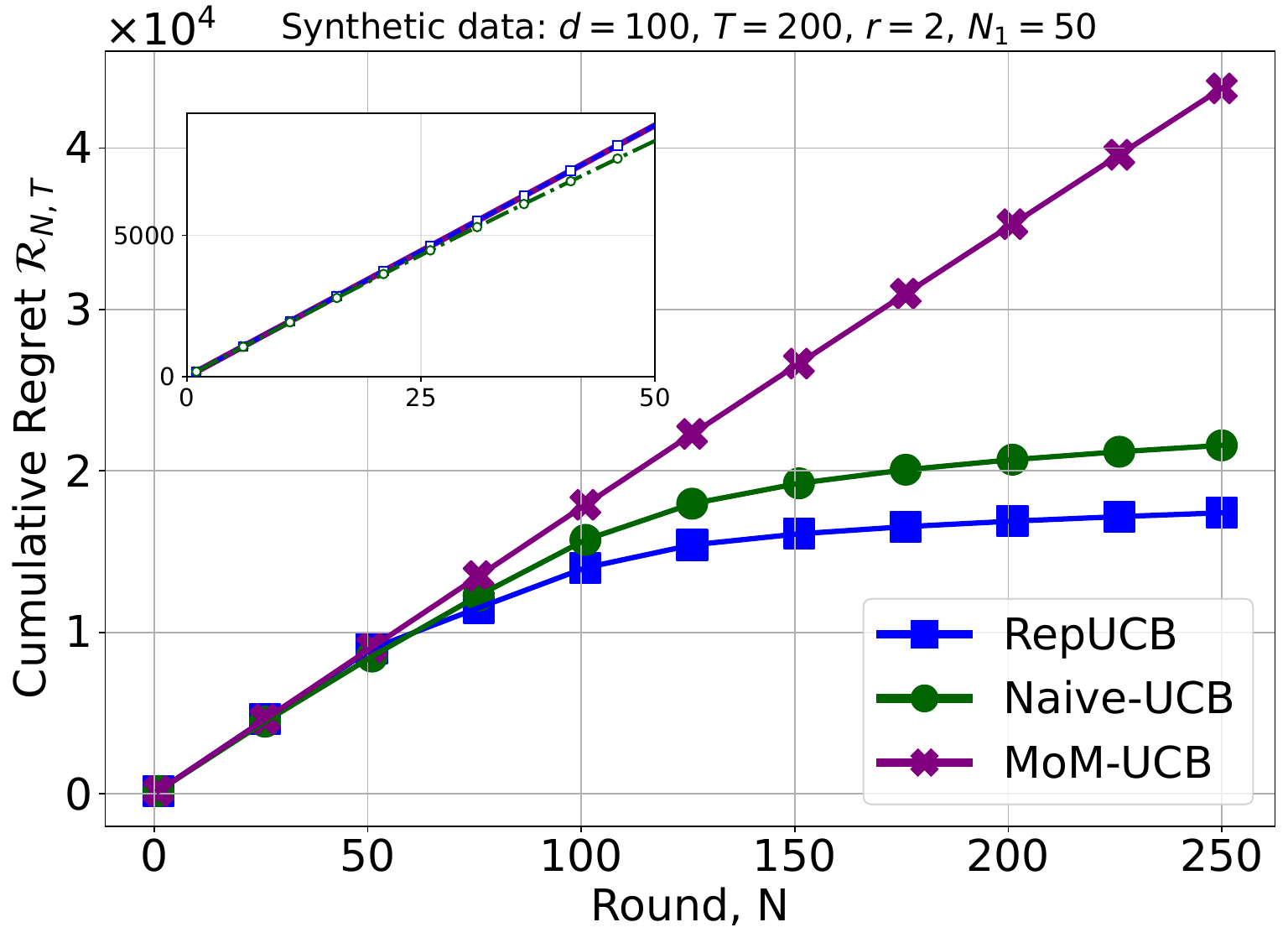}\vspace{-2 mm}}\hspace{1.2 em}%
\subcaptionbox{\label{fig:1}}{\includegraphics[scale=0.15]{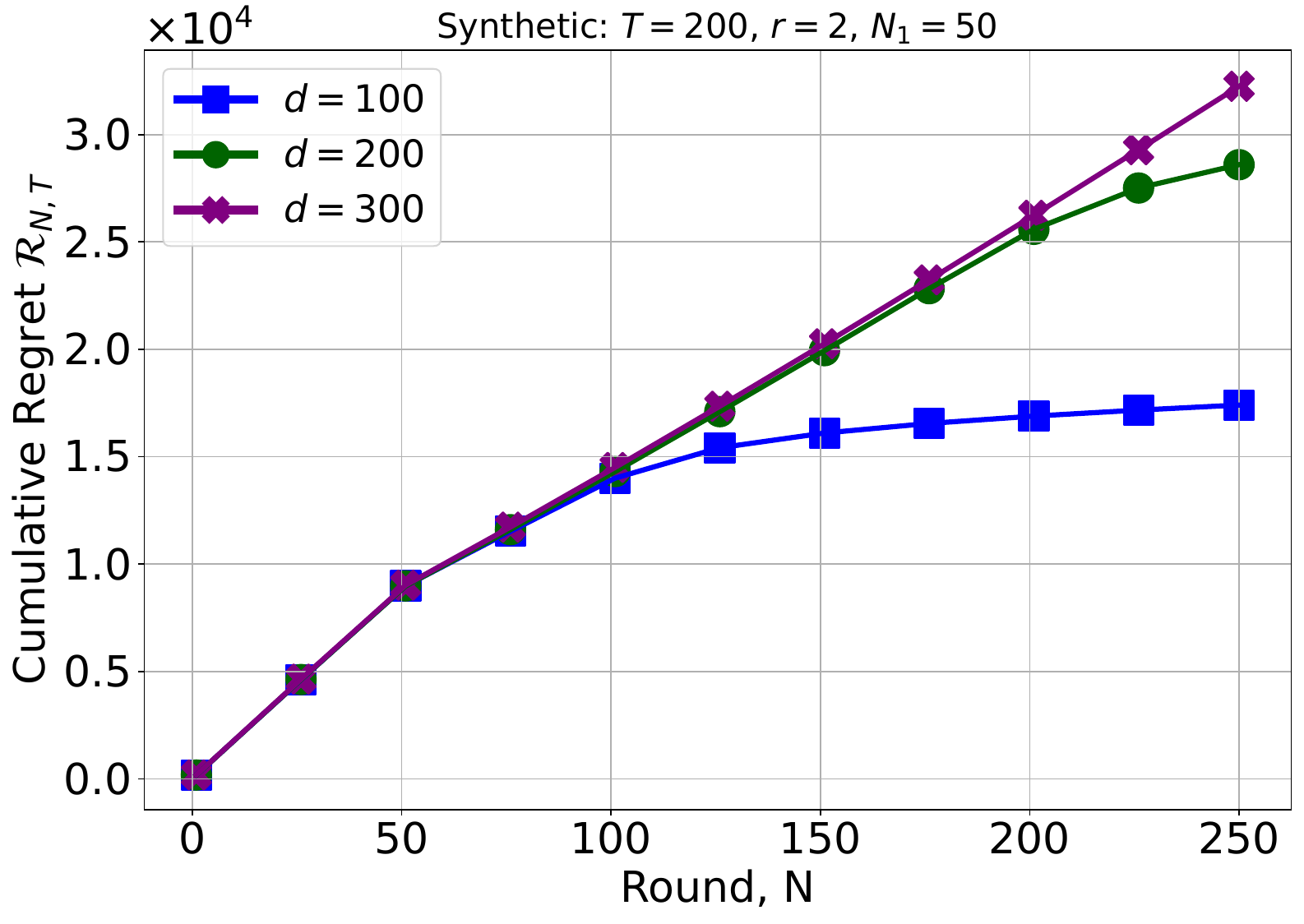}\vspace{-2 mm}}\hspace{1.2 em}%
\subcaptionbox{\label{fig:2}}{\includegraphics[scale=0.15]{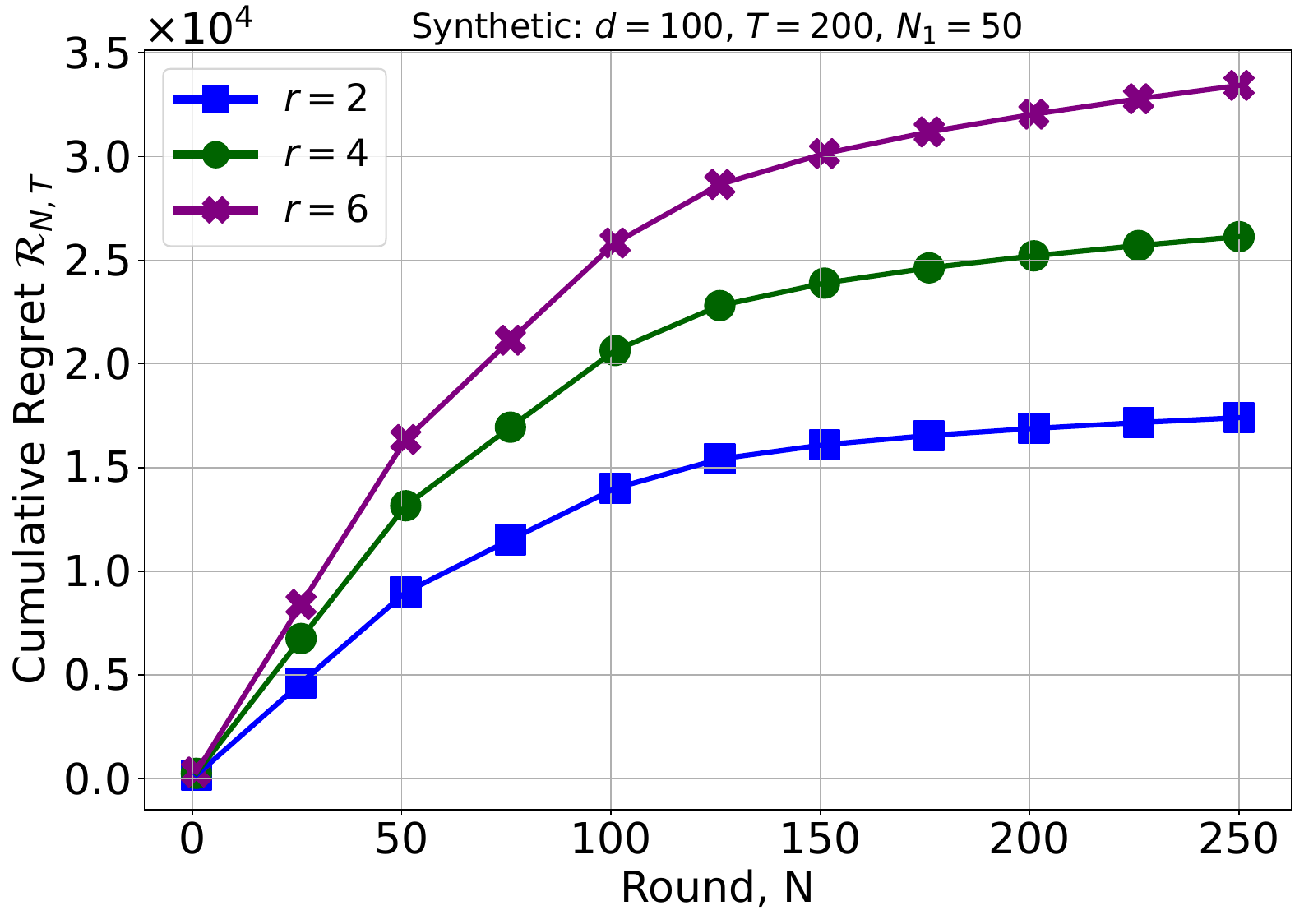}\vspace{-2 mm}}\hspace{1.2 em}%
\subcaptionbox{\label{fig:3}}{\includegraphics[scale=0.15]{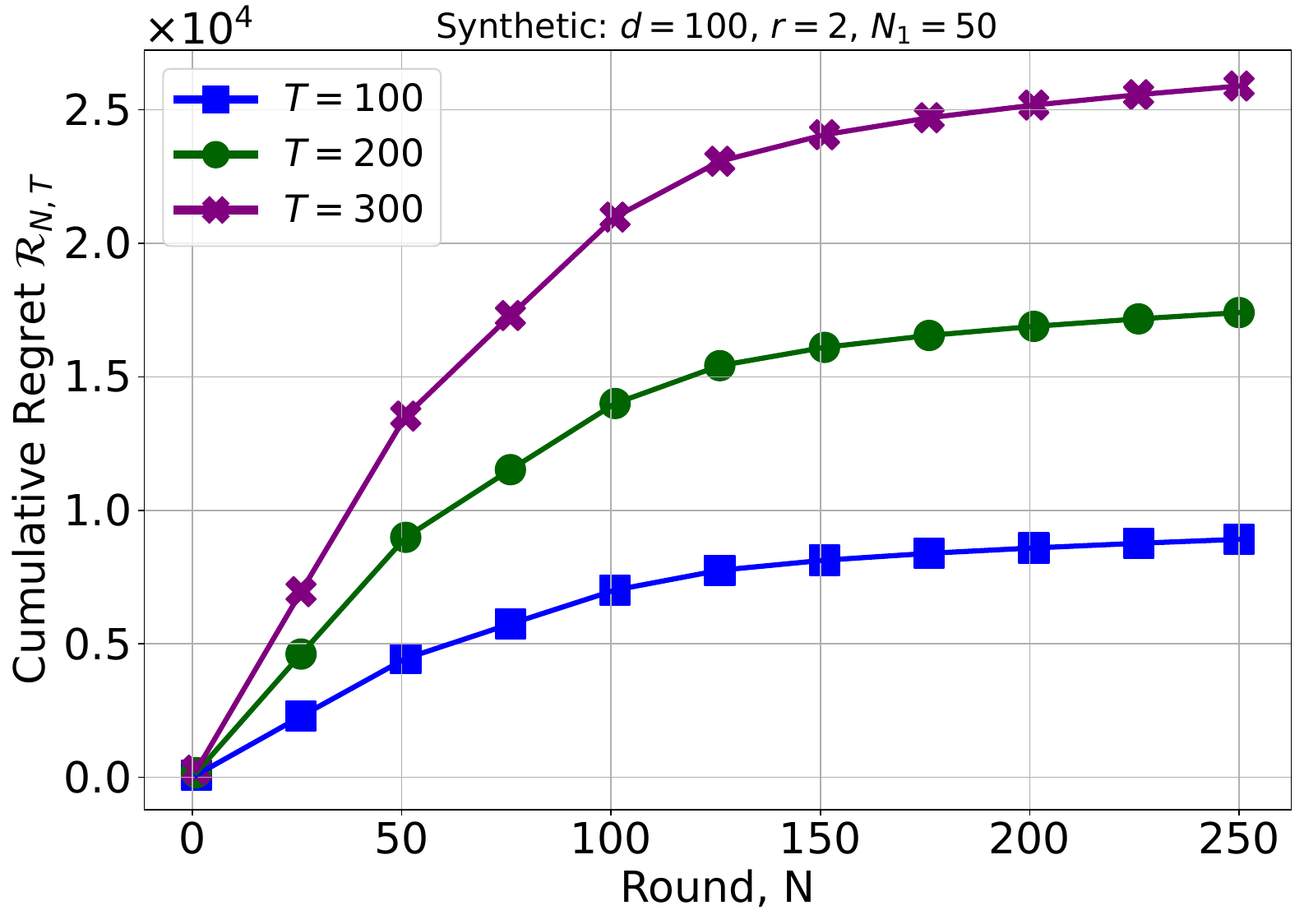}\vspace{-2 mm}}\hspace{1 em}
\vspace{-4mm}
\caption{\footnotesize Plots present the cumulative regret vs. round. Figure~\ref{fig:0} compares our proposed (RepUCB) algorithm against two benchmark approaches (MoM-UCB and Naive-UCB). Figure~\ref{fig:1} presents a plot for different feature dimensions $d \in \{100, 200, 300\}$. Figure~\ref{fig:2} presents a plot for different ranks $r \in \{2, 4, 6\}$. Figure~\ref{fig:3} presents a plot for different numbers of tasks $T \in \{100, 200, 300\}$. The default parameters are set as $d = 100, T = 200, r = 2, N_1 = 50, N = 250$.}\label{fig:plot}
\end{figure*}
\noindent{\bf Remark:} Theorem~\ref{l1} ensures that the true reward parameter $\thetats$ consistently lies within the confidence set with high probability. After exploration phase, we maintain the estimated shared representation $\wB$ unchanged while updating an $r$-dimensional coefficient vector $\ww_\nt$. For this, we  proposed an  estimator $\ww_\nt = V_\nt^{-1} \sum_{m=N_1+1}^n \wB^\top x_\mt y_\mt$ for each task utilizing the projected features $\wB^\top x_\nt$. The observation noise term in Eq.~\eqref{confidence interval} therefore depends on the $r \times r$ matrix $V_\nt$. In contrast to the single-task confidence bound presented in Theorem~2 of~\cite{abbasi2011improved}, this term has a $\sqrt{r}$ dimensional dependency, instead of $\sqrt{d}$. This clarifies the benefit of learning the shared model when $r\ll d$.
Given that $\wB$ is an estimate of $B^\star$, its column space may not exactly cover $\thetats$. The confidence set therefore remains in $\bR^d$ and includes a tolerance for the representation estimation error. In Eq.~\eqref{confidence interval} , this error is included via the $\delta_0$-dependent part of the regularization term and the additional term $2\sqrt{n-N_1}L\delta_0\mu\sqrt{\frac{r}{T}}\sigma_{\max}^\star$. The latter considers the impact of representation error on the total online observations. A lower estimation error bound $\delta_0$ on $\wB$ reduces both factors. 
Thus, a more accurate estimate of the shared representation leads to improved task-specific parameter estimates and, consequently, lower overall regret.
%
\subsection{Regret Analysis}\label{sec:regret}
In this subsection, we present the cumulative regret bound.
\begin{theorem} \label{t1}
Suppose Assumptions~\ref{assume:B}, \ref{assume:Distribution} and \ref{assume:incoherence} hold. Pick $\delta_0 < 0.1$ and  $\NSR := \frac{T \sigma^2}{\sigma_{\min}^{\star^2}}$ is the noise-to-signal ratio. If $N_1 T > C \mu^2 \kappa^2 L^2 d(N-N_1)\left(r\kappa^2+\NSR\right)$, then with probability at least $1-2\delta-\exp(-c(d+T))$, the cumulative regret of our proposed algorithm is bounded by
$$
\vspace{-1 mm}
\pR_\NT=\widetilde{O}\big(\sigma_{\min}^\star\sqrt{drNT}\big(\sqrt{\NSR}+(\sqrt{\lambda}+\sqrt{\frac{\lambda}{N-N_1}}+1)\kappa\big)\big).
$$
\end{theorem}
\begin{proof}
We start the analysis by bounding the cumulative regret for the random exploration step, $\pR_\NT^1$. We have
\begin{align*}
\pR_\NT^1 = \sum_{m=1}^{N_1} \sum_{t=1}^T (x_\mt^\star - x_\mt)^\top \thetats
\end{align*}
Using the Azuma-Hoeffding inequality, the union bound, Assumption~\ref{assume:incoherence} and Eq.~\eqref{eq:regret_inc}, we have that with probability at least $1-2\delta$,
\begin{align}
\pR_\NT^1\leqslant2\mu\sigma_{\max}^\star\sqrt{rN_1T\log\frac{1}{\delta}\log\frac{N_1T}{\delta}}.\label{eq:reg_1}
\end{align}
We then bound the cumulative regret $\pR_\NT^2$ for the subsequent step.
With probability at least $1-\delta$, we have
\begin{align}
\pR_\NT^2 &= \sum_{m=N_1+1}^{N} \sum_{t=1}^T x_\mt^{\star^\top} \thetats - x_\mt^\top \thetats \nonumber\\
&\leqslant \sum_{m=N_1+1}^{N} \sum_{t=1}^T x_\mt^\top \widetilde{\theta}_\mt - x_\mt^\top \thetats\label{newt1_1} \\
&=\sum_{m=N_1+1}^{N} \sum_{t=1}^T x_\mt^\top (\widetilde{\theta}_\mt - \widehat{\theta}_\mt)+x_\mt^\top (\widehat{\theta}_\mt - \thetats) \nonumber \\
&\leqslant\sum_{m=N_1+1}^{N} \sum_{t=1}^T \|\widetilde{\theta}_\mt-\widehat{\theta}_\mt\|_{\BVmtm} \|x_\mt\|_{\BVmtm^{-1}} \nonumber\\
&+ \|\widehat{\theta}_\mt-\thetats\|_{\BVmtm} \|x_\mt\|_{\BVmtm^{-1}}\nonumber\\
&\leqslant \sqrt{\sum_{m=N_1+1}^{N} \sum_{t=1}^T (\|\widetilde{\theta}_\mt-\widehat{\theta}_\mt\|_{\BVmtm}+\|\widehat{\theta}_\mt-\thetats\|_{\BVmtm})^2} \nonumber\\
&\cdot \sqrt{\sum_{m=N_1+1}^{N} \sum_{t=1}^T \|x_\mt\|_{\BVmtm^{-1}}^2}\label{newt1_2}\\
&\leqslant \scalemath{0.92}{2\sigma\sqrt{r\log\frac{1+(N-N_1)L^2/\lambda}{\delta}} \sqrt{2(N-N_1)T\sum_{t=1}^T \log\left(\frac{\det(\BVnt)}{\det(\lambda I)}\right)}} \nonumber\\
&+ \scalemath{0.9}{2(\sqrt{\lambda}+\frac{\sqrt{\lambda}}{L\sqrt{N-N_1}}+2)\mu\sigma_{\max}^\star \sqrt{2(N-N_1)r\sum_{t=1}^T \log\left(\frac{\det(\BVnt)}{\det(\lambda I)}\right)}}\label{newt1_4} \\
&\leqslant\scalemath{0.92}{\left(\sigma\sqrt{r\log\frac{1+NL^2/\lambda}{\delta}}+(\sqrt{\lambda}+\frac{\sqrt{\lambda}}{L\sqrt{N-N_1}}+2)\mu\sqrt{\frac{r}{T}}\sigma_{\max}^\star\right)}\nonumber\\
&\cdot2\sqrt{NT}\cdot\sqrt{2dT\log\left(1+\frac{(N-N_1)L^2}{\lambda}\right)} \label{eq:detVbar} \\
&\leqslant\left(\sqrt{\NSR\log\frac{1+NL^2/\lambda}{\delta}}+(\sqrt{\lambda}+\frac{\sqrt{\lambda}}{L\sqrt{N-N_1}}+2)\mu\kappa\right)\nonumber\\
&\cdot2\sigma_{\min}^\star\sqrt{2drNT}\sqrt{\log(1+\frac{NL^2}{\lambda})},\label{eq:reg_2}
\end{align}
where Eq.~\eqref{newt1_1} follows by $x_\mt^{\star^\top} \thetats \leqslant x_\mt^\top \widetilde{\theta}_\mt$. Eq.~\eqref{newt1_2} follows by $\sum_i a_i b_i \leqslant \sqrt{\sum_i a_i^2} \sqrt{\sum_i b_i^2}$. Eq.~\eqref{newt1_4} follows by Theorem~\ref{l1} and Lemma~11 in \cite{abbasi2011improved} and Eq.~\eqref{eq:detVbar} follows from applying similar approach as in Eq.~\eqref{eq:det} for $\BVnt$. Finally, Eq.\eqref{eq:reg_2} uses the definition of $NSR.$ Combining the bound for $\pR_\NT^1$ and $\pR_\NT^2$ and using the union bound, with probability at least $1-2\delta-\exp(-c(d+T))$, 
\begin{align*}
\pR_\NT &= \pR_\NT^1 + \pR_\NT^2 \leqslant \mathrm{Eq.~\eqref{eq:reg_1}+ Eq.~\eqref{eq:reg_2}} \\
&=\widetilde{O}\Big(\sigma_{\min}^\star\sqrt{drNT}\big(\sqrt{\NSR}+(\sqrt{\lambda}+1)\kappa\big)\Big).
\end{align*}
\end{proof}
\noindent{\bf Remark:} We compare the regret bound in Theorem~\ref{t1} with the baseline approach to validate the effectiveness of our algorithm. Treating $\sigma_{\max}^{\star}$, $\sigma_{\min}^{\star}$, and $\lambda$ as constants, the cumulative regret is bounded by $\widetilde{O}\left(\sqrt{drNT}(\sqrt{\NSR}+1)\right)$. In comparison, if we use the OFUL approach from \cite{abbasi2011improved} to independently solve $T$ tasks, the cumulative regret is $\widetilde{O}(dT\sqrt{N})$. Given that $r \ll \min\{d, T\}$, and the noise-to-signal ratio ($\NSR$) is assumed to be small, this comparison demonstrates the efficiency of our algorithm, illustrating a significant improvement in cumulative regret, specifically when the number of tasks $T$ is large. 
%
\section{Simulations}
\label{sec:Simulations}
This section evaluates the efficacy of our proposed algorithm. We set the parameters as follows: $d = 100$, $T = 200$, $r = 2$, $N_1 = 50$, $N = 250$, $\delta = 10^{-3}$, $\lambda = 1$. We considered two benchmark approaches: (i)~a naive approach (Naive-UCB) that independently solves $T$ tasks using OFUL from the first round \cite{abbasi2011improved}, and (ii)~a MoM-based approach (MoM-UCB) using the estimator in \cite{collins2021exploiting, tripuraneni2021provable, yang2020impact}. The MoM estimator estimates the representation matrix by calculating the top-$r$ singular vectors of $\widehat{\Theta}=\frac{1}{N_1T}\sum_{n=1, t=1}^{N_1, T} y_\nt^2 x_\nt x_\nt^\top$. The MoM-UCB differs from RepUCB only in the representation estimation step and uses the same subsequent OFUL procedure and parameters. 

Our proposed algorithm consists of two main phases: an exploration phase followed by an OFUL phase. To simulate the true representation matrix $B^\star \in \mathbb{R}^{d \times r}$, we generate a $d \times r$ matrix with i.i.d. standard Gaussian entries and orthonormalize its columns. The task-specific parameter matrix $W^\star \in \mathbb{R}^{r \times T}$ is generated with i.i.d. standard Gaussian entries. In each round, each task is given a set of candidate feature vectors $x_\nt$.
The additive noise is i.i.d. Gaussian with zero mean and variance $0.01$. During this phase, each task randomly selects actions from its candidate set and obtains the corresponding rewards, which are subsequently used to estimate the representation matrix $\wB$. Upon learning the estimate $\wB$, the algorithm proceeds to the OFUL phase. In each round, each task constructs its own confidence ellipsoid and chooses an optimistic action-parameter pair that maximizes the expected reward, with the action selected from its candidate set. The parameters are then updated based on the chosen action and the observed rewards. All results are averaged across $100$ independent trials. 

Figure~\ref{fig:0} presents the cumulative regret plots comparing the three algorithms. The inset in the upper-left corner indicates that RepUCB and MoM-UCB achieve same cumulative regret in the exploration phase, $N_1=50$ rounds. This is because they perform the same randomly chosen actions during the exploration phase. In contrast, Naive-UCB performs OFUL from the first round and achieves lower cumulative regret in this first phase. However, during the OFUL phase of our proposed algorithm, i.e., rounds $N_1+1, \ldots, N$, our proposed algorithm learns more efficiently and achieves a lower cumulative regret than both benchmark approaches. Figures~\ref{fig:1}, \ref{fig:2}, and \ref{fig:3} present the cumulative regret plots for our approach after varying the feature dimension, the rank, and the number of tasks, respectively. The figures indicate that the cumulative regret increases as the feature dimension $d$ or rank $r$ increases. The cumulative regret (summed over all tasks) also increases with the number of tasks $T$. Most importantly, regret of our proposed approach scales in a sublinear fashion.

\vspace{-2 mm}
\section{Conclusion}
\label{sec:Conclusion}
In this paper, we introduced RepUCB, a novel OFUL algorithm using multi-task representation learning for linear bandits. 
By leveraging the underlying low-dimensional representation of the reward parameters, we proposed a new approach to construct the confidence set such that the reward parameter lies in it with high probability. 
Using the above confidence set, we proposed a UCB algorithm based on the OFUL principle. We proved the regret guarantees for the proposed approach and evaluated its performance via numerical simulations and compared with baseline algorithms. In our experiments, we demonstrate that the proposed approach consistently outperforms the baseline methods and achieves sublinear regret scaling with respect to the problem size.
\pagebreak
\begingroup
\bibliographystyle{IEEEbib}
\bibliography{Bandits}

@article{schotthofer2025dynamical,
  title={Dynamical low-rank compression of neural networks with robustness under adversarial attacks},
  author={Schotth{\"o}fer, Steffen and Yang, H Lexie and Schnake, Stefan},
  journal={arXiv preprint arXiv:2505.08022},
  year={2025}
}

@article{bose2025lore,
  title={LoRe: Personalizing LLMs via Low-Rank Reward Modeling},
  author={Bose, Avinandan and Xiong, Zhihan and Chi, Yuejie and Du, Simon Shaolei and Xiao, Lin and Fazel, Maryam},
  journal={arXiv preprint arXiv:2504.14439},
  year={2025}
}

@article{wang2024lora,
  title={Lora-ga: Low-rank adaptation with gradient approximation},
  author={Wang, Shaowen and Yu, Linxi and Li, Jian},
  journal={Advances in Neural Information Processing Systems},
  volume={37},
  pages={54905--54931},
  year={2024}
}

@inproceedings{Jiabin_neurips,
  title={Provably Efficient Multi-Task Meta Bandit Learning via Shared Representations},
  author={Lin, Jiabin and Moothedath, Shana},
  booktitle={Conference on Neural Information Processing Systems},
  year={2025}
}

@inproceedings{tripuraneni2021provable,
  title={Provable meta-learning of linear representations},
  author={Tripuraneni, Nilesh and Jin, Chi and Jordan, Michael},
  booktitle={International Conference on Machine Learning},
  pages={10434--10443},
  year={2021},
  organization={PMLR}
}

@article{abbasi2011improved,
  title={Improved algorithms for linear stochastic bandits},
  author={Abbasi-Yadkori, Yasin and P{\'a}l, D{\'a}vid and Szepesv{\'a}ri, Csaba},
  journal={Advances in Neural Information Processing Systems},
  volume={24},
  pages={2312--2320},
  year={2011}
}

@article{yang2020impact,
  title={Impact of representation learning in linear bandits},
  author={Yang, Jiaqi and Hu, Wei and Lee, Jason D and Du, Simon S},
  journal={arXiv:2010.06531},
  year={2020}
}

@article{matcomp_candes,
  author    = "E. J. Candes and B. Recht",
  title     = "Exact matrix completion via convex optimization",
  journal   = "Found. of Comput. Math",
  volume    = "9",
  year      = "2008",
  pages     = "717-772"
}

@inproceedings{du2020few,
  title={Few-shot learning via learning the representation, provably},
  author={Du, Simon S and Hu, Wei and Kakade, Sham M and Lee, Jason D and Lei, Qi},
  booktitle={ International Conference on Learning Representations (ICLR)},
  year={2021}
}

@inproceedings{hu2021near,
  title={Near-optimal representation learning for linear bandits and linear rl},
  author={Hu, Jiachen and Chen, Xiaoyu and Jin, Chi and Li, Lihong and Wang, Liwei},
  booktitle={International Conference on Machine Learning},
  pages={4349--4358},
  year={2021}
}

@inproceedings{cella2023multi,
  title={Multi-task representation learning with stochastic linear bandits},
  author={Cella, Leonardo and Lounici, Karim and Pacreau, Gr{\'e}goire and Pontil, Massimiliano},
  booktitle={International Conference on Artificial Intelligence and Statistics},
  pages={4822--4847},
  year={2023},
}

@article{thekumparampil2021sample,
  title={Sample efficient linear meta-learning by alternating minimization},
  author={Thekumparampil, Kiran Koshy and Jain, Prateek and Netrapalli, Praneeth and Oh, Sewoong},
  journal={arxiv:2105.08306},
  year={2021}
}

@inproceedings{collins2021exploiting,
  title={Exploiting shared representations for personalized federated learning},
  author={Collins, Liam and Hassani, Hamed and Mokhtari, Aryan and Shakkottai, Sanjay},
  booktitle={International conference on machine learning},
  pages={2089--2099},
  year={2021},
  organization={PMLR}
}

@article{singh2024noisy,
  title={Noisy Low Rank Column-wise Sensing},
  author={Singh, Ankit Pratap and Vaswani, Namrata},
  journal={arxiv:2409.08384},
  year={2024}
}

@InProceedings{OurICML,
  title = {Fast and Sample Efficient Multi-Task Representation Learning in Stochastic Contextual Bandits},
  author = {Lin, Jiabin and Moothedath, Shana and Vaswani, Namrata},
  booktitle = {Proceedings of the 41st International Conference on Machine Learning},
  pages = {30227--30251},
  year = {2024},
  volume = {235},
  series = {Proceedings of Machine Learning Research},
  month = {21--27 Jul},
  publisher = {PMLR}
}

@inproceedings{chen2022active,
  title={Active multi-task representation learning},
  author={Chen, Yifang and Jamieson, Kevin and Du, Simon},
  booktitle={International Conference on Machine Learning},
  pages={3271--3298},
  year={2022},
  organization={PMLR}
}

@article{qin2022non,
  title={Non-stationary representation learning in sequential linear bandits},
  author={Qin, Yuzhen and Menara, Tommaso and Oymak, Samet and Ching, ShiNung and Pasqualetti, Fabio},
  journal={IEEE Open Journal of Control Systems},
  volume={1},
  pages={41--56},
  year={2022},
  publisher={IEEE}
}

\end{document}